\documentclass[journal]{IEEEtran}
\usepackage{cite}

\ifCLASSINFOpdf
\else
\fi
\usepackage[cmex10]{amsmath}

\usepackage[linesnumbered,ruled]{algorithm2e}

\usepackage{array, tabularx}

\usepackage{amsthm}

\usepackage{url}

\usepackage{float}
\usepackage{graphicx}
\graphicspath{ {images/} }

\usepackage{algpseudocode}

\newtheorem{theorem}{THEOREM}[section]
\newtheorem{lemma}[theorem]{LEMMA}

\begin{document}
%
\title{TCM-ICP: Transformation Compatibility Measure for Registering Multiple LIDAR Scans}
%
%
%

\author{Aby~Thomas,~\IEEEmembership{}
        Adarsh~Sunilkumar,~\IEEEmembership{}
        Shankar~Shylesh,~\IEEEmembership{}
        Aby~Abahai~T.,~\IEEEmembership{}
        Subhasree~Methirumangalath,~\IEEEmembership{}
        Dong~Chen~\IEEEmembership{}
        and~Jiju~Peethambaran
\thanks{Aby~Thomas, Department
of Computer Science and Engineering, National Institute of Technology, Calicut, Kerala, India}
\thanks{Adarsh~Sunilkumar, Department
of Computer Science and Engineering, National Institute of Technology, Calicut, Kerala, India }
\thanks{Shankar Shylesh,  Department
of Computer Science and Engineering, National Institute of Technology, Calicut, Kerala, India}
\thanks{Aby Abahai T, Department
of Computer Science and Engineering, National Institute of Technology, Calicut, Kerala, India}
\thanks{Subhasree~Methirumangalath, Associate Professor, Department
of Computer Science and Engineering, National Institute of Technology, Calicut, Kerala, India}
\thanks{Dong~Chen, Associate Professor,
College of Civil Engineering,Nanjing Forestry University, Nanjing, China}
\thanks{Jiju~Peethambaran, Assistant Professor,
Department of Math \& Computing Science, Saint Mary’s University, Halifax, Canada}
}

\maketitle

\begin{abstract}
Rigid registration of multi-view and multi-platform LiDAR scans is a fundamental problem in 3D mapping, robotic navigation, and large-scale urban modeling applications. Data acquisition with LiDAR sensors involves scanning multiple areas from different points of view, thus generating partially overlapping point clouds of the real world scenes. Traditionally, ICP (Iterative Closest Point) algorithm is used to register the acquired point clouds together to form a unique point cloud that captures the scanned real world scene. Conventional ICP faces local minima issues and often needs a coarse initial alignment to converge to the optimum. In this work, we present an algorithm for registering multiple, overlapping LiDAR scans. We introduce a geometric metric called Transformation Compatibility Measure (TCM) which aids in choosing the most similar point clouds for registration in each iteration of the algorithm. The LiDAR scan most similar to the reference LiDAR scan is then transformed using simplex technique. An optimization of the transformation using gradient descent and simulated annealing techniques are then applied to improve the resulting registration. We evaluate the proposed algorithm on four different real world scenes and experimental results shows that the registration performance of the proposed method is comparable or superior to the traditionally used registration methods. Further, the algorithm achieves superior registration results even when dealing with outliers.
\end{abstract}


%
\IEEEpeerreviewmaketitle

\section{Introduction}
%
%
%
%

Over the last decade, Light Detection and Ranging (LiDAR) systems have emerged as a predominant tool for capturing outdoor scenes for various applications such as 3D mapping and navigation, urban modeling and simulation, and risk assessment of urban utilities such as powerlines. LiDAR sensors mounted on various platforms including aircraft or unmanned aerial vehicles (UAV), vehicle, satellite and tripod \cite{b12} collect the 3D data of outdoor scenes in the form of point clouds. Considering the operational efficiency and quality of scans, multiple scans of the same object or scene is acquired from various positions which helps in removing blind spots in the scan. Multiple scans are then stitched together to generate the final 3D scan of the scene. In this paper, we focus on a hybrid solution for registering multiple LiDAR scans using statistical and geometrical constraints.
 
In general, point cloud registration is the process of assigning correspondences between two sets of points and recovering the transformation that maps one point set to the other \cite{b8}. Point cloud registration frequently applies a coarse-to-fine registration strategy\cite{b2}. In coarse registration, the initial registration parameters for the rigid body transformation of two point clouds are mainly estimated using the feature-based method\cite{b3}. Registration based on point, line and surface features are included in the feature-based coarse registration method. In fine registration, the main aim is to achieve maximum overlap of two point clouds, primarily using the iterative approximation method, normal distribution transform method, random sample consensus method or methods with auxiliary data\cite{b3}.

Registration of noisy and overlapping LiDAR data of three-dimensional surfaces or scenes is a challenging problem\cite{b2}. The quality of data provided is a key factor that affects correctness of point cloud registration. The data obtained through LIDAR is sparse and non-homogeneous. All LIDAR registration mechanisms work on the principle of aligning the key points of the object scanned so as to combine multiple point clouds from various views in order to recreate a 3D model of the object. The presence of irrelevant and inconsistent data can thus cause complications in the registration process. Such irrelevant data points are commonly referred to as outliers. Thus removing outliers from the obtained input becomes vital in creating an efficient and accurate model. Only data that has been well preprocessed is expected to yield promising results.\par

Many point cloud registration algorithms have been developed over the years. Iterative approximation method is widely used in the field of point cloud registration among the fine registration methods. Iterative approximation methods mainly refers to ICP (Iterative Closest Point) algorithm proposed by Mckay and Besel\cite{b1}. However the iterative nature of the ICP algorithm makes it less efficient in dealing with high density and large-scale point cloud scenes and also slow at finding corresponding points between two point clouds \cite{b3}. In order to overcome the problems of existing LIDAR registration methods such as need of auxiliary data or target, requirement of sufficient overlapping areas, and difficulty in feature extraction and matching an algorithm is developed called TCM ICP (Transformation Compatibility Measure for Registering Multiple LIDAR Scans). We make the following key contributions.\\
 \begin{itemize}
   \item A geometrical and statistical metric called Transformation Compatibility Measure (TCM) which is employed to select the point cloud that is most compatible with the reference/destination cloud for registration.
   \item A simplex algorithm based registration method for LIDAR data that takes multiple LIDAR scans (point clouds) of a given object or scene having variable densities and coordinate frames as input and generate a 3D model of the object or scene by transforming the various point clouds into a global coordinate system.
   \item An optimization method by combining gradient descent and simulated annealing techniques with the transformation matrix calculation to improve the results. 
 \end{itemize}




\section{Related Work}
LiDAR point cloud data sets tend to have non-uniform point distribution and differing coordinate axes that lead to challenges in registering the points to 3D models. Various coarse registration and fine registration techniques have been developed over the years for registering LiDAR point clouds. In general,point cloud registration can be classified into two: feature based registration and iterative approximation methods such as ICP algorithm. The removal outliers is one among the important steps of point cloud registration.\par


\subsection{Feature-based Registration}
Feature-based transformation models such as point-based, line-based and plane-based 3D transformation models were studied and improved by people across the globe. Over the years, these techniques have been used extensively for point cloud registration. The accuracy of registered points depends on the techniques used for feature extraction\cite{b3}. Feature based registration is a point cloud registration mechanism in which key features of the object scanned such as corners and edges are extracted from the point clouds and mapping of corresponding pairs of features among point clouds is carried out. Jaw and Chuang (2008)\cite{b4} described a feature based transformation technique that uses point-based, line-based and plane-based techniques. This 3D transformation was mathematically described by a 7-parameter spatial similarity transformation\cite{b4}. The technique showed promising results with high degrees of flexibility and accuracy when working with datasets of nominal sizes. The transformations are modeled mathematically and parameters are approximated using techniques such as least square distance approximation. The method takes a pair of point clouds, apply necessary transformations and combine the two point clouds according to matching results. The main drawback of the technique comes from the fact that taking random point clouds and joining their transformed results lead to propagation of errors to neighbouring point clouds. The method performs better than those which use a single feature for registration. Even though it provides a certain amount of robustness, the results proved to be inferior to those of other iterative methods of registration.\par

Forstner and Khoshelham\cite{b9} proposed a method for point cloud registration based on plane to plane correspondences. This method works by extracting planar regions from point clouds, the extraction process uses maximum likelihood estimation to extract planes with some degree of uncertainty. Three direct solutions namely direct algebraic solution, direct whitened algebraic solution and single-iteration ML-solution were applied on the extracted features in order to guarantee convergence of the feature matching process. These methods are accurate and time efficient to some extent but prove to be statistically suboptimal compared to some iterative approximation methods.

\subsection{ICP-based Registration}

A major challenge in registration of LIDAR point clouds using iterative approximation methods is finding a fast and efficient method for obtaining matching point pairs and devising a feasible
algorithm for translation and rotation of one point cloud onto another reference point cloud. This process is vital in transforming all the point clouds to a single and global coordinate system\cite{b13}.
Only after performing necessary preprocessing, transformation and error checks the point clouds can be registered to produce a 3D model. \par

Besl (1992)\cite{b1} introduced an algorithm called iterative
closest point algorithm. Given an initial guess of the rigid body transformation, the algorithm proved to be successful in a variety of applications related to aligning 3D models. The ICP algorithm fixes one point cloud as the reference point cloud and then runs iteratively on the other point clouds (source), through each iteration the source point cloud is transformed so as to match the reference point cloud. A root mean square distance between the points is used to align each point in the source point cloud to its match found in the reference point cloud. The process is stopped when the error metric (usually calculated as sum of distances between matched points) is within some threshold value or some predefined number of iteration is reached.
Even though the method put forward by Besl gave promising results, it came with some drawbacks. It uses a point-to-point distance method for finding corresponding point pairs in the point clouds and
removing them. This proved to be less efficient when dealing with datasets of large size. The preprocessing steps employed by this method also proved to be incapable of removing non-matching points and outliers from the point clouds. The need for an initial guess can also be considered as a drawback of this algorithm.\par

Xin and Pu (2010)\cite{b1} examined the drawbacks of the existing ICP and came up with an improved ICP algorithm. This method used the center of gravity of the matching pairs of as the reference point. This reference point and a combination of orientation constraints are together used to remove false point pairs. Point pair distance constraints used in this method also improved the performance. This modified algorithm uses the same preprocessing step as the conventional ICP algorithm. It introduces improvements in the second step of the algorithm, instead of a point to point distance method this new algorithm uses point pair distance constraints and centre of gravities of the point clouds as the reference pairs to reject false pairs.\par
The more the points pairs, better is the output quality of the improved algorithm put forward in \cite{b1}. Fewer number of point pairs leads to failure in registration. Higher number of erroneous or non-matching point pairs leads to higher error rates and improper transformation of point clouds. Matching pairs have high sensitivity to noise and have a high percentage of false pairs at the early stages of aligning. The accuracy of registration, speed and convergence rates are the points to be improved in this version of ICP.\par
Go-ICP \cite{b14} is a global optimization method on the well-established Branch-and-Bound (BnB) theory. However, selecting an appropriate domain parametrization to construct a tree structure in BnB and, more importantly, to extract efficient error boundaries based on parametrization .In order to address the local minima problem,global registration methods have been investigated in Go-ICP. Here local ICP is integrated into the BnB scheme, which speeds up the new method while guaranteeing global optimality.\par
It is also possible to accelerate the closest point search using fast global registration algorithm \cite{b16} that does not involve
iterative sampling, model fitting, or local refinement. The algorithm does not require initialization and can align noisy partially overlapping surfaces. It optimizes a robust objective defined densely over the surfaces. Due to this dense coverage, the algorithm
directly produces an alignment that is as precise as that computed by well-initialized local refinement algorithms. The optimization does not require closest-point queries in the inner loop.  Another formulation of the ICP algorithm is registration using sparsity inducing norm\cite{b17} that optimizes by avoiding difficulties such as sensitivity to outliers and missing data.Also PointNet which represents point clouds itself can be thought of as a learnable imaging function\cite{b18}. Here classical vision algorithms for image alignment can be applied
for point cloud registration.\par
\subsection{Effect of outliers for registration}
The presence of outliers or irrelevant data in the dataset can lead to improper matching of data points among the point clouds during the process of calculating the transformation matrices. This inturn leads to lower accuracy and coarse edges in the 3D model that is generated\cite{b14}. Thus outlier detection and removal becomes vital in producing good results. The most popular outlier detection techniques are distance based outlier detection technique, density based outlier detection technique and cluster based outlier detection technique\cite{b5}.
Distance based outlier detection techniques work by calculating distances between data points, if a point has a distance close to its nearest neighbour then it is considered as a normal point, otherwise it is marked as an outlier.
In density based outlier detection algorithm, every object in the data set is referred as local outlier factor (LOF). The LOF is the degree which is assigned to the object of data set. It is also defined as the ratio between the local density of an object and the average of those of its k nearest neighbors\cite{b5}.
In clustering based outlier detection, the given data points are clustered into groups, similar or neighbouring data points are expected to end up in the same cluster\cite{b5}. Many clustering algorithms are used for clustering with K-means algorithm, which being one of the most common choices. Some statistical methods or techniques such as weighted centre based methods are applied on each cluster to detect and remove outliers\cite{b6}.\par

\section{Methodology}
The input to our multi scan registration system is a set if point clouds, $\Psi=\{P_{1}, P_{2}, .., P_{n}\}$. The proposed model for LiDAR registration (referred to as TCM ICP) has three steps. The first step is a preprocessing to remove outliers from the input scans. The second step consists of determining the rotational and translational matrices for each input point cloud. To this end, a reference cloud is selected based on a correspondence measure from the input point clouds. In each iteration, point cloud with the least TCM(Transformation Compatibility Measure) value is selected for the registration. Simplex and gradient optimization techniques are employed to find the optimal rotational and translational matrices for registering the selected point cloud to the reference point cloud. In the final step, the actual transformation of the point cloud to the reference frame is performed.

\begin{figure}[]
\centering
  \includegraphics[width=8cm]{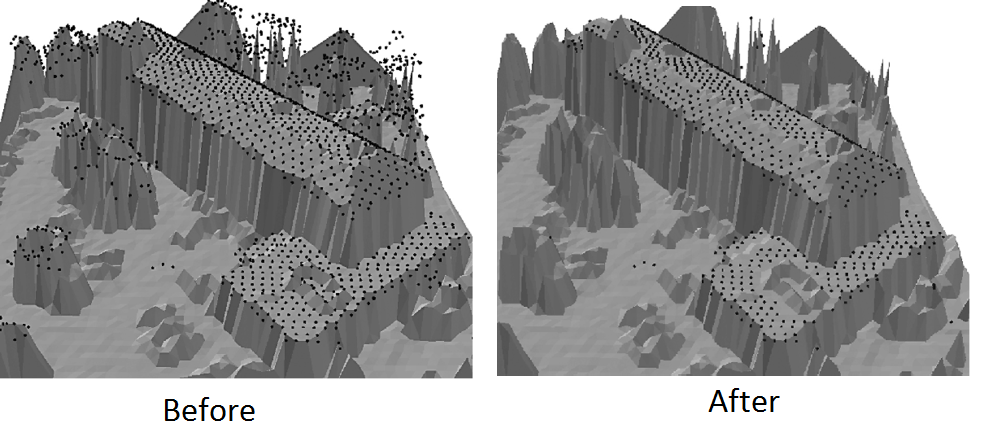}
  \caption{A sample of point cloud before and after the outlier removal }
  \label{label1}
\end{figure}

\subsection{Preprocessing}
In the preprocessing phase, we use a K-means clustering based technique to identify and remove outliers from each point cloud. First, a set of \textit{k} points are selected from each point cloud with the help of K-D tree. To this end, each point cloud is embedded into a K-D tree and a centroid point is chosen at random. All the points lying within a sphere of radius $r$ are then removed. This process is repeated until $k$ number of centroids are selected. Thus selected $k$ points are then used as initial centroids in K-means clustering. Note that, K-D tree based centroid selection ensures that every pair of centroids are spatially well separated.  Further, the use of K-D tree avoids the randomness in centroid selection and reduces the iteration count of K-means algorithm. Outlier removal is performed using the $k$ clusters generated by K-means, i.e., points that lie beyond a fixed threshold limit from each cluster centroid are removed from the point cloud. An example of outlier removal is shown in Figure \ref{label1}.

\subsection{Selection of the Reference Point Cloud}
The crucial step in the registration process is finding a rigid transformation that aligns the input point cloud to the reference point cloud. A reference point cloud is one of the input point cloud that exhibits a good correspondence with all the remaining point clouds. Linear inequalities are formulated for all other source point clouds using the underlying concepts of simplex method\cite{b11}. These inequalities are then solved to obtain the transformation matrices for each point cloud which is discussed in section \ref{mea}.

Reference point cloud has a significant influence on the quality of registration. Good correspondence between the reference cloud and the remaining point clouds is essential and would greatly improves the computational performance in finding optimal transformations.  To define the reference point cloud, the correspondence from a point cloud to every other point cloud is calculated. The correspondence of a point cloud $P_{i}$ to another point cloud $P_{j}$ is a closeness measure between $P_{i}$ and $P_{j}$. A threshold is imposed on correspondence values, i.e., if the correspondence value between $P_{i}$ and $P_{j}$ is less than a threshold, then a connection is said to exist between $P_{i}$ and $P_{j}$. The point cloud with the highest number of direct connections is chosen as the reference point cloud which will define the reference coordinate system for the final registration. If several point clouds meet this condition, the point cloud located in a central position regarding the list of files is arbitrarily defined as the reference data, assuming that the data have been acquired successively in the spatial distribution.
\begin{figure}[]
  \centering
  \includegraphics[scale = 0.3]{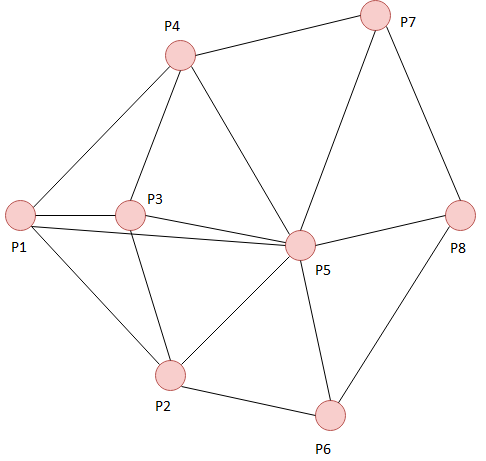}
  \caption{An example of correspondence graph representing connections between point clouds. Note that the algorithm chooses node 7, which has the highest as the reference point cloud for the registration. }
  \label{label1}
\end{figure}

\subsection{Transformation Compatibility Measure ($\tau$)}\label{mea}
 The order in which the point clouds are merged using simplex method can drastically affect the accuracy of the final output. We propose a metric called the Transformation Compatibility Measure(denoted by $\tau$) to impose an ordering in the point cloud selection and merging. TCM measure (Equation \ref{eq1}) captures the compatibility among different point clouds which is quantified through the inter- and intra-cluster distances, and the normalized distance between the point clouds. TCM measure ensures that the most similar point clouds are selected for transformation in each iteration of the registration process. This selection process is pivotal in determining the correct order and values of the point cloud transformations.
\begin{equation}\label{eq1}
\tau(P_{i},P_{j})= f_{ij}g_{ij}h_{ij}
\end{equation}

The inter-cluster distance between two point clouds $P_{i}$ and $P_{j}$ is computed by adding the minimum distances from each point of one cloud to cluster centroid of the other point cloud as given in Equations \ref{eq3}-\ref{eq5}.
\begin{equation}\label{eq3}
f_{ij}=X_{i}+Y_{j}
\end{equation}
\begin{equation}\label{eq4}
X_{i}=min(\parallel p_{k}-\mathcal{C}(P_{j})\parallel^2 \mid \forall p_{k}\in P_{i}, i \neq j)
\end{equation}
Here $\mathcal{C}(.)$ denotes the centroid of the point set.
\begin{equation}\label{eq5}
Y_{j}=min(\parallel p_{k}-\mathcal{C}(P_{i})\parallel^2 \mid \forall p_{k}\in P_{j},  i \neq j)
\end{equation}

We can observe that Equations \ref{eq3}-\ref{eq5} is a customized version of the Hausdorff metric (Equation \ref{he}).
\begin{equation}\label{he}
H(A,B)=max{(h(A,B),h(B,A))}
\end{equation}
which defines the Hausdorff distance between two geometric shapes $A$ and $B$. The two distances $h(A, B$) and $h(B, A)$ are sometimes termed as forward and backward Hausdorff distances of A to B.
\begin{equation}
h(A,B)=\underset{a\in A}{max}(\underset{b\in B}{min}(\parallel a-b\parallel^2))
\end{equation}
where $a$ and $b$ are points of sets $A$ and $B$ respectively. Lower the distance value, better is the matching between $A$ and $B$. The most important use of Hausdorff metric is that it can be used to check if a certain feature is present in a point cloud or not. This method gives interesting results, even in the presence of noise or occlusion (when the target is partially hidden). In our case, the presence of noise is very minimum due to initial preprocessing.

The intra-cluster distance is obtained by calculating the minimum of all point to point distance in each cloud and subtracting one cloud value from the other (Equations \ref{eq6}-\ref{eq8}).
\begin{equation}\label{eq6}
g_{ij}=M_{i}-N_{j}
\end{equation}
\begin{equation}\label{eq7}
M_{i}=min(\parallel p_{k}-p_{l}\parallel^2 \mid \forall p_{k}, p_{l}\in P_{i}, k \neq l)
\end{equation}
\begin{equation}\label{eq8}
N_{j}=min(\parallel p_{k}-p_{l}\parallel^2 \mid \forall p_{k}, p_{l}\in P_{j}, k \neq l)
\end{equation}

Finally, to avoid the bias which may induce due to the varying number of points in point clouds, we perform a normalization of the measure using Equation \ref{eqn}.
\begin{equation}\label{eqn}
h_{ij}=\frac{1}{|P_i||P_j|}
\end{equation}

We use the method of contradiction to to establish the correctness of the transformation compatibility measure (refer to Lemma \ref{lem1}).
\begin{lemma}\label{lem1}
Transformation compatibility metric (Equation \ref{eq1}) between the point clouds correctly aids in choosing the most similar point clouds in each iteration.
\end{lemma}
\begin{proof}
The proof is by contradiction. We assume that TCM does not choose the most similar point clouds for transformation. During $m^{th}$ iteration (in general any $m$) TCM chooses point clouds $P_i$ and $P_j$ as most similar point clouds instead of $P_i$ and $P_k$  which are actually the most similar in the $m^{th}$ iteration.\\
\[\tau(P_i,P_j)<= \tau(P_i,P_k)\]

TCM has two main components-first is the customized version of Hausdorff distance and the other component is the difference between the intra-cluster distances between the two point clouds. As per our assumption, the difference between the intra-cluster should be lower for $P_k$ and $P_i$ as compared to $P_i$ and $P_j$. That is first part of TCM should be higher for $P_i$ and $P_k$ pair. This implies that the customized Hausdorff distances between the most similar point cloud is large, which is obviously not correct. Further, the point clouds are obtained after KD tree based clustering which implies outliers have been removed, indicating that\\
 $\tau(P_i ,P_j) > \tau(P_i, P_k)$.\\

Point clouds with least value of TCM are chosen during each iteration. Since $\tau(P_i, P_k)$ is lower than $\tau(P_i ,P_j)$, the algorithm would have originally chosen $P_i$ and $P_k$ for transformation. Thus our assumption that the wrong point cloud pair was chosen by TCM is wrong and hence, the proof.
\end{proof}

\subsection{Registration of Point Clouds}
\paragraph{Simplex Algorithm for Registration}Consider two point clouds $P_{i}$ and $P_{j}$, our objective is to find the rigid transformation that aligns $P_{j}$ to $P_{i}$. The transformation of one point cloud to another involves translation, $T$ as well as rotation, $R$. Given $P_{i}$ and $P_{j}$, computing the unknown $T$ and $R$ effectively is the key aspect of point cloud registration. A point cloud $P $transformed into $P'$ can be expressed by the Equation \ref{rt}.
\begin{equation}\label{rt}
P'=R*P+T
\end{equation}
Solving this equation yields the rotation and translation matrices.

In our problem, the number of constraints that can be formulated is less than the number of variables present and as a consequence, most methods for solving linear equations fails to work here. The problem is thus modeled as an optimization problem with three possibilities.
\[P'>R*P+T\]
\[P'<R*P+T\]
\[P'=R*P+T\]
The values of $R$ and $T$ with least error obtained by optimizing these inequalities using simplex method \cite{b11} are used as the rotational and translational matrices for a point cloud. Though algebraic in nature, the underlying concepts of simplex algorithm are geometric. Simplex algorithm tests adjacent vertices of the feasible set (which is a polytope) in sequence so that at each new vertex the objective function improves or is unchanged. The simplex method is very efficient in practice, generally taking $2m$ to $3m$ iterations at most (where $m$ is the number of equality constraints), and converging in expected polynomial time for certain distributions of random inputs. However, its worst-case complexity is exponential.


The point cloud with the best transformation, that is the point cloud that has minimum deviation from the reference point cloud is selected and merged with the reference point cloud after the transformation. This process of calculating translational and rotational matrices and merging of point clouds is continued until a single point cloud with a reference coordinate system is obtained. The merging of point clouds is done with the expectation that future iterations of the process will lead to improvement in the quality of the dataset.

\paragraph{Optimization} We combine gradient descent and simulated annealing techniques with the transformation matrix calculation to improve the results. The translation and rotation matrices calculated using simplex algorithm are fed to the gradient descent method to calculate the local optimum values for translation and rotation. Gradient descent is an optimization algorithm used to minimize the function by iteratively moving in the direction of steepest descent as defined by the negatives of the gradient. In our problem, gradient descent is used to update the parameters of transformation. Parameters refer to the coefficients in the rotational and translational matrices. All the combinations of marginal increase and marginal decrease of parameter values are performed and the combination that gives the least values of TCM (Transformation Compatibility Measure) measure after point cloud transformation are used to transform the point cloud. The pseudo code for the entire registration method is given in Algorithm \ref{hd}.


\begin{algorithm}[h]\label{alg2}
\KwIn{Pre-processed set of points clouds $\Psi=\{P_{1}, P_{2}, .. , P_{n}\}$}
\KwOut{Registered point cloud, $S$}
Construct the correspondence graph $G=(V, E)$, where $V$ is a set of vertices one for each $P_{i}\in \Psi$ and $E$ consists of all the edges $e=(v_{i}, v_{j})$ such that $correspondence(P_{i}, P_{j}) < threshold$ \;
Let $S=\{ P_{parent}\in \Psi \mid v_{parent}$ is the maximum degree vertex in $V\}$\;
Update $\Psi$, i.e., $\Psi=\Psi\setminus P_{parent}$\;
Initialize $i=2$\;
\While{$i \leq n$ } {
\For {each $P_{j} \in \Psi$} {
Compute $tcm(P_{j}, S)$\;
}
Let $P_{q}$ be the point cloud with minimum $tcm$\;
[T,R]=Simplex($P_{q}$, $S$)\;
[T$'$, R$'$]=Gradient\_descent(T, R)\;
Transform $P_{q}$, i.e., $P_{q}'$=transform($P_{q}, T', R')$\;
Update the registered point cloud $S=S\cup P_{q}'$\;
Update the point set, $\Psi=\Psi\setminus P_{q}$\;
}
\Return{$S$}\;
\caption{TCM-ICP($P$)}
\label{hd}
\end{algorithm}

\begin{figure}[H]
\centering
  \includegraphics[scale = .25]{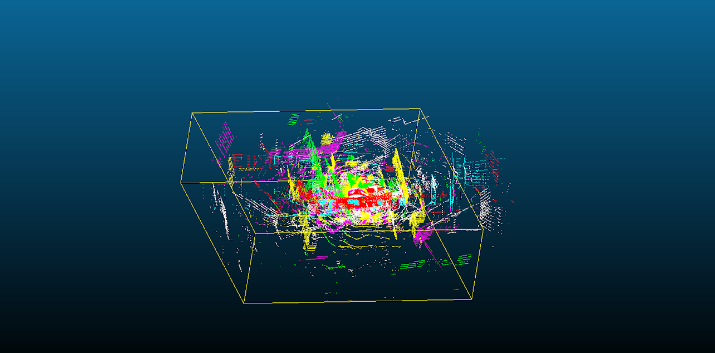}
  \caption{simple point cloud plots in various colours of Data set 1 before registration.}
  \label{bef}
\end{figure}
\begin{figure*}[ht]
	\centering
	\includegraphics[width=16cm]{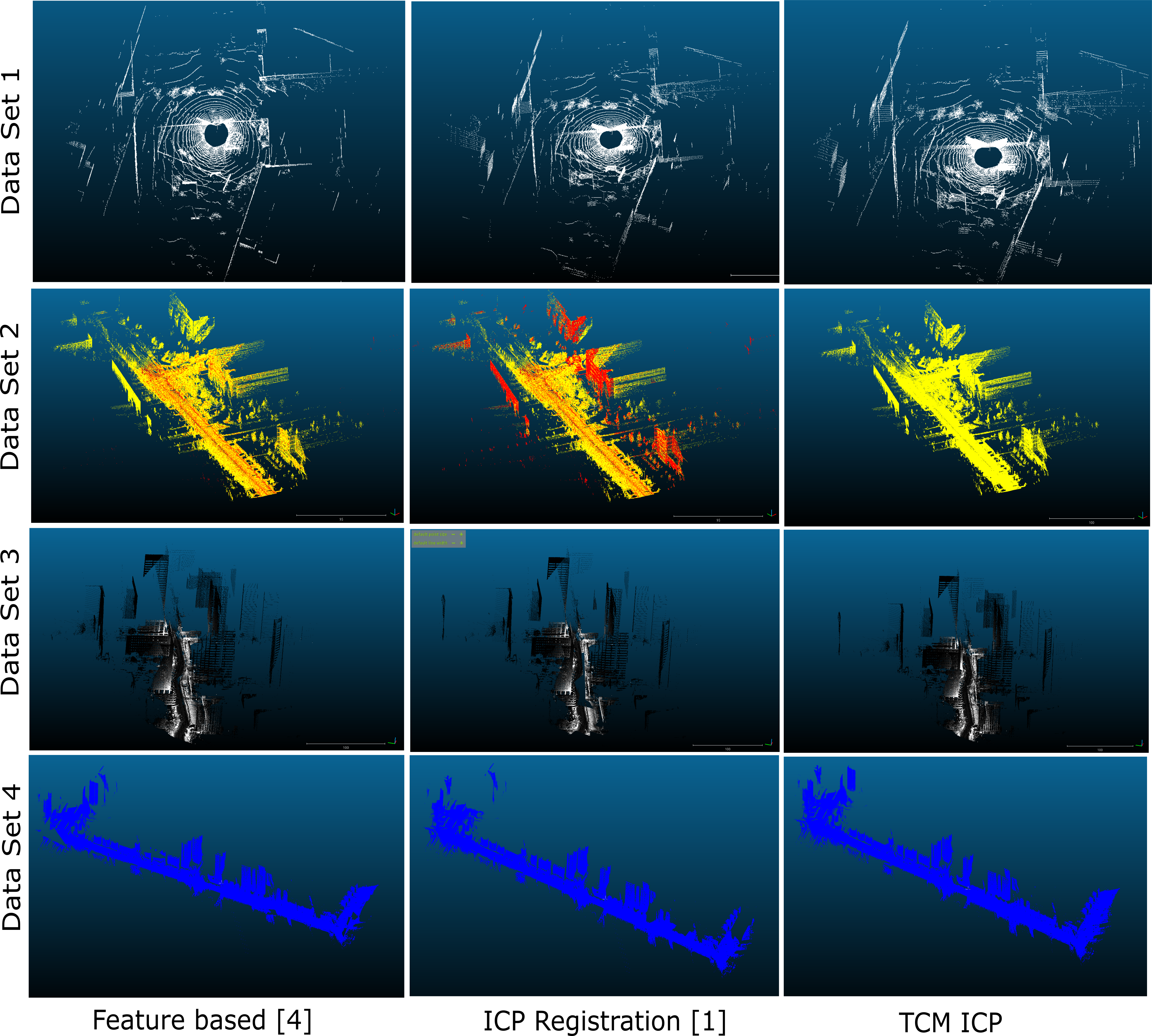}
	\caption{Gallery of registration results. The first two columns of the second row shows the overlapped images of LiDAR scans obtained after registration of Data set 2 using TCM ICP, and ICP based and feature based registration methods, respectively.  Note that the yellow scans in both the results represent the output of TCM ICP and the red represents the other result.}
	\label{result}
\end{figure*}

\section{Experiments and Results}
All the experiments were performed on a system with Intel Xeon E5-2600 processor with 2.4 GHz and 32 GB of DDR4 RAM. Software used for point cloud processing were Point Cloud Library(PCL) \cite{b19}, Computational Geometry Algorithms Library (CGAL) \cite{b20} and Cloud Compare \cite{b21}.

The point cloud registration algorithms were tested with four LIDAR data sets. 
\begin{enumerate}
	\item Data Set 1: This data set was recorded with the intention to test registration algorithm robustness in the context of navigation with low accuracy of the sensor and motion during acquisition which is localized in Clermont-Ferrand (France). It provides an urban scene consisting of buildings and trees.
	\item  Data Set 2 \& 4: This data set is a mobile LiDAR scene of road strip with building on either side.
	\item Data Set 3:This data set provides an excellent mix of the conditions that a surveying and geo-spatial firm would be challenged with logistically. This is localized in south shore of the Lynx Lake in Prescott, United States, which is ideal because of the large amount of assets to map and an open tree canopy.
\end{enumerate}
  \textnormal{All the data sets stated before were registered using conventional ICP based method \cite{b1}, feature based method \cite{b4} and the proposed TCM ICP. All the four data sets consists of mobile LiDAR of various cities. The subsequent evaluation comprise of qualitative and quantitative analysis of the outputs. These data sets contain sparse as well as dense areas. Further, the input scans contain isolated regions and overlaping areas. So the data sets are heterogeneous in nature and hence, represent a good choice for testing the proposed registration method.}

\subsection{Qualitative Analysis}
Figure \ref{bef} is a plot consist of all the point clouds of Data set 1 before registration. Individual point clouds are shown in different colours. Figure \ref{bef} clearly shows that pure 3D plot of the scanned points is not meaningful with many misalignments between the individual scans. A simple visual inspection of Figure \ref{result}, suggests that majority of the point clouds have been correctly placed in the reference coordinate system. The first row of Figure \ref{result} shows the outputs of feature based, ICP and TCM ICP registrations applied on Data set 1. Since  TCM ICP algorithm not susceptible to isolated points, it identifies the overlapped regions and works well with both sparse and dense data sets. Hence the registration of Data Set 1 is done with minimum error compared to other methods as quantified in row 3 of Table \ref{table1}.

The results given in the second row of Figure \ref{result} shows the Data set 2 registered using various point cloud registration methods. The column 1, row 2 of Figure \ref{result} is an overlap image of 3D scans obtained by registration using feature based method and TCM ICP. It is clear from this figure that some features such as the buildings towards the end were not properly registered by feature based registration method. The second figure in the second row of Figure \ref{result} shows the overlapped image of 3D scans obtained by registration using ICP based method and TCM ICP. The points in red towards the outer edges of this figure represents outliers present in ICP based registration result. These outliers have been successfully removed by TCM ICP. Similarly, Third and fourth rows of Figure \ref{result} shows outputs of registration of Data Set 3 and 4 using feature based, ICP, TCM ICP based registration methods, respectively. The registered 3D scan is that of a city street. All three registration methods produce outputs of comparable quality, but on a closer look, it can be seen that TCM ICP produced more accurate results as evident from Section \ref{qunt}.

\begin{table}[h]
  \begin{center}
      \caption{Performance of various registration algorithms on Data Set 1}
    \label{table1}
    \begin{tabular}{l|c|r|l}
      \textbf{Criteria} & \textbf{TCM ICP} & \textbf{ICP} & \textbf{Feature based}\\
      \hline
      \textbf{Time(min)} & 3716 & 3700 & \textbf{3510}\\
      \textbf{Iterations} & 2940 & 3015 & \textbf{2081}\\
      \textbf{RMS (50000 points)} & \textbf{1.022} & 1.88 & 1.34\\
      \textbf{Error due to isolated points} & \textbf{1.58} & 2.02 & 2.75\\
      \textbf{Error due to feature blurring} & \textbf{2.85} & 3.02 & 2.95\\
    \end{tabular}
  \end{center}
\end{table}

\subsection{Quantitative Analysis}\label{qunt}
\paragraph{Criteria for Analysis}The registration results using different methods are compared using various criteria including time, number of iterations, RMS (Root Mean Square) value of 50000 points, error after adding isolated points, error after removing points, error after feature blurring, cloud-to-cloud distance and standard deviation. The sensitivity of the algorithm is measured by feature blurring, i.e., the process of adding additional points to conceal the features in the point cloud.
\begin{table}[h]
  \begin{center}
    \caption{Performance of various registration algorithms on Data Set 2}
    \label{table2}
    \begin{tabular}{l|c|r|l}
      \textbf{Criteria} & \textbf{TCM ICP} & \textbf{ICP} & \textbf{Feature based}\\
      \hline
      \textbf{Time(min)} & 4208 & \textbf{3542} & 4050\\
      \textbf{Iterations} & 3528 & 4538 & \textbf{2938}\\
      \textbf{RMS (50000 points)} & \textbf{1.09} & 2.87 & 1.25\\
      \textbf{Error due to isolated points} & \textbf{1.82} & 2.9 & 3.85\\
      \textbf{Error due to feature blurring} & 3.74 & 4.53 & \textbf{3.22}\\
    \end{tabular}
  \end{center}
\end{table}

\paragraph{Comparison} A tabular comparison of time, number of iterations, RMS (50000 points), error after isolated point addition and error after feature blurring of Data Set 1 is shown in Table \ref{table1}. TCM ICP showed better results compared to ICP based and feature based registration when exposed to noise, occlusion of features and feature blurring. However TCM ICP relapsed in terms of time taken for completing the process of registration. ICP and feature based algorithms are implemented using the libraries like FLANN (Fast Library for Approximate Nearest Neighbours)\cite{b7}, which consist of most optimal heuristic implementations.
Table \ref{table2}-\ref{table3} reports various performance attributes of the compared algorithms on Dat Set 2 and 3. The results obtained are similar to that of Data Set 1, where TCM ICP outperforms the other two registration methods in terms of noise error, RMS and feature blurring, but trails behind in the case of computational time for the overall registration process. This is because of Simplex and Gradient descent method applied for optimization in Algorithm\label{alg2}.
\begin{table}[h]
  \begin{center}
     \caption{Performance of various registration algorithms on Data Set 3}
    \label{table3}
    \begin{tabular}{l|c|r|l}
      \textbf{Criteria} & \textbf{TCM ICP} & \textbf{ICP} & \textbf{Feature based}\\
      \hline
      \textbf{Time(min)} & 5316 & \textbf{4339} & 4875\\
      \textbf{Iterations} & 4250 & 4037 & \textbf{3405}\\
      \textbf{RMS (50000 points)} & \textbf{2.13} & 3.83 & 2.68\\
      \textbf{Error due to isolated points} & \textbf{2.53} & 3.99 & 3.98\\
      \textbf{Error due to feature blurring} & \textbf{3.85} & 4.30 & 3.96\\
    \end{tabular}
  \end{center}
\end{table}

\begin{figure*}[!h]
\centering
  \includegraphics[width=15cm]{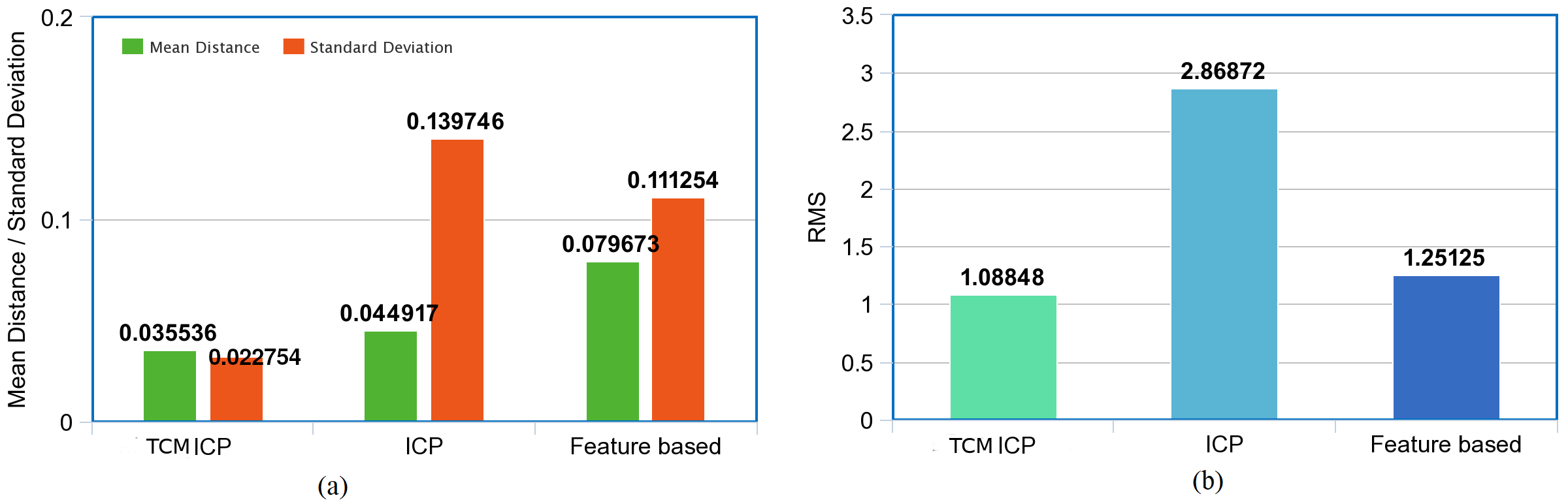}
  \caption{Bar chart showing the performance comparison of various registration methods with respect to (a) mean cloud-to-cloud distances and standard deviations and (b) RMS values of Data Set 2.}
 \label{bar2}
\end{figure*}

\begin{figure*}[!h]
\centering
  \includegraphics[width=15cm]{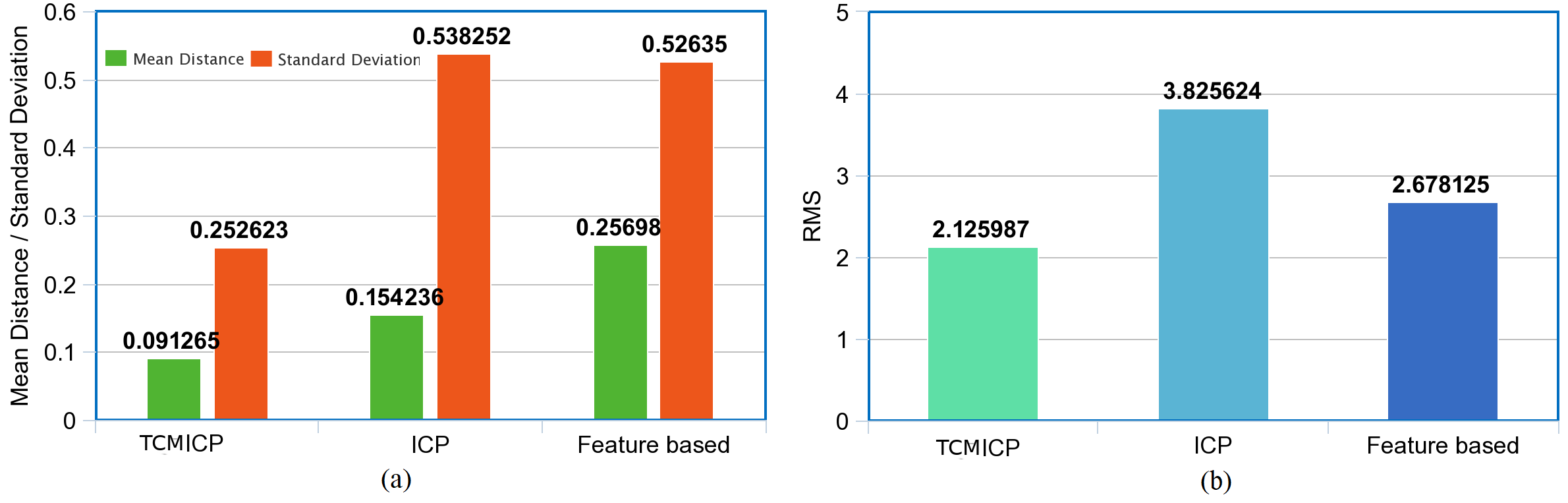}
  \caption{Bar chart showing the performance comparison of various registration methods with respect to (a) mean cloud-to-cloud distances and standard deviations and (b) RMS values of Data Set 3.}
 \label{bar3}
\end{figure*}

\paragraph{Ground truth based Evaluation} Bar charts presented in Figures \ref{bar2} and \ref{bar3} show that TCM ICP performs better as compared to ICP based and feature based registrations in terms of mean distance and standard deviation of the registered scans from the corresponding ground truths of Data Set 2 and 3. The RMS value of TCM ICP is also lower than that of ICP based and feature based registration for both Data Set 2 and 3 as shown in Figures \ref{bar2} and \ref{bar3}. Though the RMS values of TCM ICP and feature based registration are comparable, TCM ICP slightly scores over the feature based registration (refer to Figures \ref{bar2}(b) \& \ref{bar3}(b)). It is evident that the high quality registration generated by TCM ICP is mainly due to the prepossessing for outlier removal and a well designed point cloud selection method.
\begin{figure*}[h]
\centering
  \includegraphics[width=17cm]{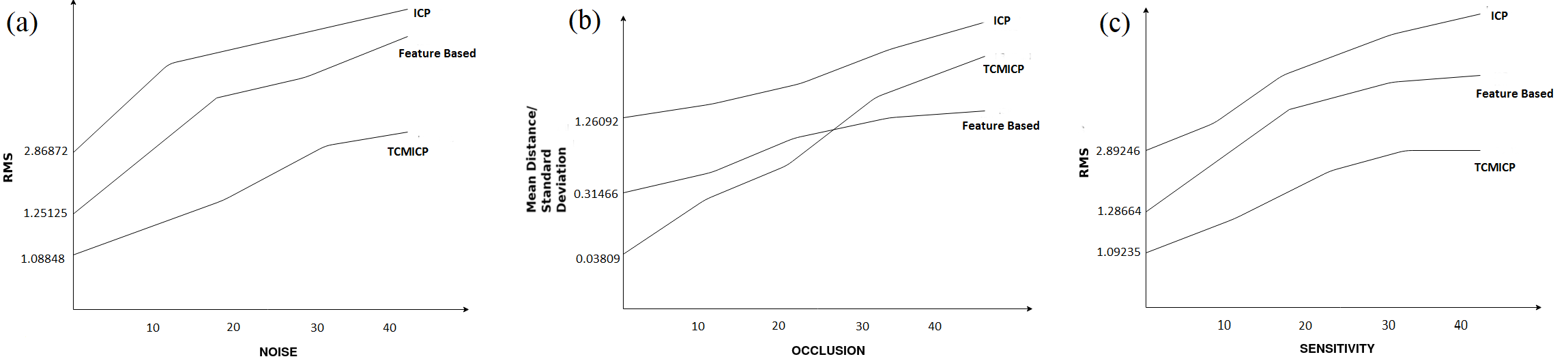}
  \caption{Graph showing the effect of (a) noise, (b) occlusion and (c) sensitivity of various registration algorithms on Data Set 2.}
  \label{gra2}
\end{figure*}

%
%

\begin{figure*}[h]
\centering
  \includegraphics[width=16cm]{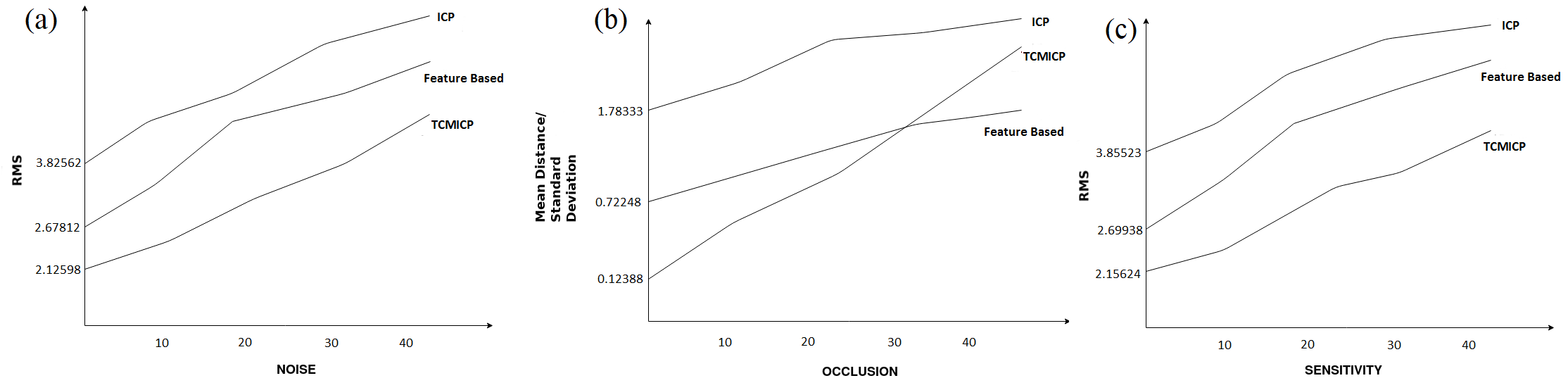}
  \caption{Graph showing the effect of (a) noise, (b) occlusion and (c) sensitivity of various registration algorithms on Data Set 3.}
  \label{gra3}
\end{figure*}

%
%
\paragraph{Robustness to Defect Laden Scans}To evaluate the performance of TCM ICP for defect laden LiDAR scans, we performed another experiment by manually introducing the artifacts such as noise, occlusions and sparsity. Graphs plotted in Figures \ref{gra2} and \ref{gra3} represent the RMS values or sum of mean distance and standard deviation between clouds against percentage of noise added, percentage of occlusion of features and percentage of points removed(sensitivity). Addition of noise, occlusion of features and removal of points cause a monotonic increase in RMS value for the test data, i.e., Data set 2 and 3. It is clear that RMS values of TCM ICP are always lower than that of ICP based and feature based registrations for noise addition and point removal experiments. However, in the case of sum of mean distance and standard deviation versus percentage of occlusion of features, the curve representing feature based registration overtakes TCM ICP after a relatively high percentage of features are removed.  This pattern is observed in both, Dataset 2 and 3. ICP based registration fails to produce the expected result because ICP suffers from local minima issues and it is always expect the clouds to have sufficient initial alignment before registration. This is the reason why a coarse initial alignment is provided before applying ICP\cite{b12}.

\section{Conclusion}
We proposed an algorithm for registering multiple and partially overlapping LIDAR scans of an outdoor scene. The key enabler of the proposed multi-view registration is a geometrical and statistical measure called transformation compatibility measure which effectively identifies the point cloud most similar to the reference scan in each iteration of the algorithm. The method overcome the problems of existing LIDAR registration methods such as the need for auxiliary data or target, requirement of sufficient overlapping areas, and difficulty in feature extraction and matching. The suggested method for point cloud registration shows promising results when tested against traditionally used methods using LiDAR scans of various outdoor scenes. The method also works effectively in removing outliers from the input point clouds, which in turn enhances the accuracy of the registered 3D scans. 

In the future, further step is to be taken to integrate data sets from heterogeneous sources, e.g., airborne, mobile and terrestrial. The execution time of the method is a factor that needs to be looked upon further. Using custom point cloud data structures and libraries for processing the point clouds would cut down the run time by a large factor. Data driven Techniques such as geometric deep learning\cite{b13} and various fine tuning methods may be used to improve the performance of the registration process.

\section*{Acknowledgment}
We acknowledge the developers of PCL and Dr. Gilles Debunne, developer of QGLViewer who provided invaluable insights in using PCL and CGAL libraries.\par

\ifCLASSOPTIONcaptionsoff
  \newpage
\fi




\begin{thebibliography}{1}

\bibitem{b1} Wei~Xin, Jiexin~Pu, \emph{An Improved ICP Algorithm for Point Cloud Registration},  International Conference on Computational and Information Sciences, 2010.
\bibitem{b2} P.~Dong, Q.~Chen. \emph{LiDAR Remote Sensing and Applications}. Boca Raton, Florida, United States: CRC Press, 2018.
\bibitem{b3} L.~Cheng, S.~Chen, X.~Liu, Hao~Xu, Yang~Wu , Manchun~Li and Y.~Chen, \emph{Registration of Laser Scanning Point Clouds: A Review}, in Sensor, vol. 18(5), 1641; doi:10.3390/s18051641, 2018.
\bibitem{b4} J.J.~Jaw and T.Y.~Chuang, \emph{Feature-Based Registration of terrestrial LIDAR point clouds}, International Archives of Photogrammetry, Remote Sensing and Spatial Information Sciences, 2010.
\bibitem{b5} H.C.~Mandhare, S.R.~Idate, \emph{A comparative study of cluster based outlier detection, distance based outlier detection and density based outlier detection techniques}, International Conference on Intelligent Computing and Control Systems, 2017.
\bibitem{b6} J.J.~Manoharan, S.H.~Ganesh, J.G.R.~Sathiaseelan , \emph{Outlier detection using enhanced K-means Clustering Algorithm and weight based senter approach}, International Journal of Computer Science and Mobile Computing, 2016.
\bibitem{b7}
Point Cloud Library Documentation, \url{http://http://docs.pointclouds.orgs}. Last accessed 28 Jan 2019
\bibitem{b8} A.~Myronenko , X.~Song, \emph{Point Set Registration: Coherent Point Drift}, IEEE Transactions on Software Engineering, 2010.
\bibitem{b9} W.~Forstner, K.~Khoshelham, \emph{Efficient and Accurate Registration of Point Clouds with Plane to Plane Correspondences},  IEEE International Conference on Computer Vision Workshops (ICCVW), 2017.
\bibitem{b10} R.~A.~Brown, \emph{Building a Balanced k-d Tree in O(kn log n) Time}, Journal of Computer Graphics Techniques, 2015.
\bibitem{b11} F.~S.~Hillier, G. ~J.~Lieberman, \emph{Introduction to operations research}, McGraw Hill, New York, 2001.
\bibitem{b12} Dirk~Holz, Alexandru~E.~Ichim, Federico~Tombari, Radu~B.~Rusu, and Sven~Behnke, \emph{Registration With the Point cloud Library PCL}, IEEE Robotics \& Automation Magazine, Volume 22, Issue 4, pp. 110-124, December 2015.
\bibitem{b13} Gil~Elbaz, Tamar~Avraham and Anath~Fischer, \emph{3D Point Cloud Registration for Localization using a Deep Neural Network Auto-Encoder},  2017 IEEE Conference on Computer Vision and Pattern Recognition (CVPR).
\bibitem{b14} Jiaolong~Yang, Hongdong~Li2 and Yunde~Jia, \emph{Go-ICP: Solving 3D Registration Efficiently and Globally Optimally}, 2013 IEEE International Conference on Computer Vision.
\bibitem{b15} Jiaolong~Yang, Hongdong~Li, Dylan~Campbell and Yunde~Jia, \emph{Go-ICP: A Globally Optimal Solution to 3D ICP Point-Set Registration},  IEEE Transactions on Pattern Analysis and Machine Intelligence ( Volume: 38 , Issue: 11 , Nov. 1 2016 ).
\bibitem{b16} Timothee Jost and Heinz Hugli,\emph{ Fast ICP Alogorithms for the Registration of 3D Data},Pattern Recognition-24th DAGM Symposium Zurich, Switzerland,pp 91-99,September 16–18, 2002 Proceedings.
\bibitem{b17} Sofien Bouaziz, Andrea Tagliasacchi and Mark Pauly,\emph{Sparse Iterative Closest Point},Volume 32 (2013), Number 5,Eurographics Symposium on Geometry Processing 2013.
\bibitem{b18} Yasuhiro Aoki, Hunter Goforth, Rangaprasad Arun Srivatsan and Simon Lucey,\emph{PointNetLK: Robust and Efficient Point Cloud Registration using PointNet},Open Access Version(2019),Computer Vision Foundation.    
\bibitem{b19} Point Cloud Library, http://pointclouds.org/ [last accessed: 01-04-2020]
\bibitem{b20} Computational Geometry Algorithms Library, https://www.cgal.org/ [last accessed: 12-15-2019]
\bibitem{b21} Cloud Compare, https://www.danielgm.net/cc/, [last accessed: 01-04-2019]
\end{thebibliography}
%

\end{document}